\documentclass{article} 
\usepackage{fullpage}
\usepackage{amssymb}
\usepackage{amsmath}
\usepackage{amsthm}

\usepackage[round]{natbib}
\usepackage[colorlinks=true, linkcolor=blue]{hyperref}

\usepackage{microtype}
\usepackage{graphicx}
\usepackage{booktabs}
\usepackage{algorithm}

\usepackage{enumitem}
\usepackage{pgfplots}
\usepackage{subcaption}
\usepackage{cleveref}
\usepackage{authblk}
\usepackage{mathtools} 

\usepackage{color}
\usepackage{enumitem}

\newtheorem{thm}{Theorem}[section]
\newtheorem{lemma}[thm]{Lemma}
\newtheorem{cor}[thm]{Corollary}
\newtheorem{prop}[thm]{Proposition}

\newtheorem*{thm*}{Theorem}
\newtheorem*{lemma*}{Lemma}
\newtheorem*{cor*}{Corollary}
\newtheorem*{prop*}{Proposition}
\newtheorem*{conjecture*}{Conjecture}

\theoremstyle{definition}

\newtheorem*{defn*}{Definition}
\theoremstyle{definition}

\theoremstyle{definition}

\theoremstyle{remark}
\newtheorem*{ex*}{Example}
\theoremstyle{definition}

\theoremstyle{definition}
\newtheorem*{assm*}{Assumption}
\theoremstyle{remark}

\theoremstyle{remark}
\newtheorem*{remark*}{Remark}

\DeclareFontFamily{U}{mathx}{\hyphenchar\font45}
\DeclareFontShape{U}{mathx}{m}{n}{
      <5> <6> <7> <8> <9> <10> gen * mathx
      <10.95> mathx10 <12> <14.4> <17.28> <20.74> <24.88> mathx12
      }{}
\DeclareSymbolFont{mathx}{U}{mathx}{m}{n}
\DeclareMathSymbol{\intop}  {1}{mathx}{"B3}

\newcommand{\wt}{\widetilde}
\newcommand{\wh}{\widehat}

\newcommand{\gap}{\Delta}
\newcommand{\ind}{1}

\let\temp\phi
\let\phi\varphi
\let\varphi\temp

\renewcommand{\sec}{\textsection}

\newcommand{\pr}{\mathbb{P}}

\newcommand{\R}{\mathbb{R}}

\newcommand{\E}{\mathbb{E}}

\newcommand{\normalN}{\mathcal{N}}
\newcommand{\given}{\,|\,}  

\newcommand{\symdiff}{\triangle}

\DeclareMathOperator{\KL}{KL}

\DeclareMathOperator*{\argmax}{arg\,max}

\DeclareMathOperator{\BinomialDist}{Bin}

\newcommand{\dens}{F}
\newcommand{\denscmp}{f}

\newcommand{\mix}{\Lambda}
\newcommand{\wgt}{\lambda}

\newcommand{\truemix}{\mix^{*}}
\newcommand{\truewgt}{\wgt^{*}}
\newcommand{\truedens}{\dens^{*}}
\newcommand{\truedenscmp}{\denscmp^{*}}

\newcommand{\estmix}{\wh{\mix}}
\newcommand{\estwgt}{\wh{\wgt}}

\newcommand{\estdenscmp}{\wh{\denscmp}}

\renewcommand{\Xi}{X^{(i)}}
\newcommand{\Yi}{Y^{(i)}}

\newcommand{\baseX}{\mathcal{X}}
\newcommand{\baseY}{\mathcal{Y}}
\newcommand{\class}{\alpha}
\newcommand{\alldens}{\mathcal{P}}
\newcommand{\mixing}{\mathcal{M}}

\newcommand{\wass}{W_{1}}
\newcommand{\coupling}{\sigma}
\newcommand{\dTV}{d_{\textup{TV}}}

\newcommand{\dec}{\mathcal{D}}
\newcommand{\truedec}{\dec^{*}}

\newcommand{\estdec}{\wh{\dec}}

\newcommand{\perm}{\pi}
\newcommand{\trueperm}{\perm^{*}}

\newcommand{\estperm}{\wh{\perm}}
\newcommand{\identperm}{\wt{\perm}}

\newcommand{\clf}{g}
\newcommand{\trueclf}{\clf^{*}}
\newcommand{\uclf}{\check{\clf}}
\newcommand{\estclf}{\wh{\clf}}

\newcommand{\nll}{\ell}

\newcommand{\estmle}{\estperm_{\textup{MLE}}}
\newcommand{\gapmle}{\gap_{\textup{MLE}}}

\newcommand{\estmv}{\estperm_{\textup{MV}}}
\newcommand{\gapmv}{\gap_{\textup{MV}}}

\newcommand{\estgreedy}{\estperm_{\textup{G}}}

\title{\Large{Sample Complexity of Nonparametric Semi-Supervised Learning}}
\author[]{Chen Dan}
\author[]{Liu Leqi}
\author[]{Bryon Aragam}
\author[]{Pradeep Ravikumar}
\author[]{Eric P. Xing}
\affil[]{\emph{Carnegie Mellon University}}

\begin{document}
\maketitle


\begin{abstract}
We study the sample complexity of semi-supervised learning (SSL)
and introduce new assumptions based on the mismatch between a mixture model learned from unlabeled data and the true mixture model induced by the (unknown) class conditional distributions.
Under these assumptions, we establish an $\Omega(K\log K)$ labeled sample complexity bound without imposing parametric assumptions, where $K$ is the number of classes. Our results suggest that even in nonparametric settings it is possible to learn a near-optimal classifier using only a few labeled samples. Unlike previous theoretical work which focuses on binary classification, we consider general multiclass classification ($K>2$), which requires solving a difficult permutation learning problem. 
This permutation defines a classifier whose classification error is controlled by the Wasserstein distance between mixing measures, and we provide finite-sample results characterizing the behaviour of the excess risk of this classifier. 
Finally, we describe three algorithms for computing these estimators based on a connection to bipartite graph matching, and perform experiments to illustrate the superiority of the MLE over the majority vote estimator.
\end{abstract}

\section{Introduction}
\label{sec:intro}

With the rapid growth of modern datasets and increasingly passive collection of data, labeled data is becoming more and more expensive to obtain while unlabeled data remains cheap and plentiful in many applications. 
Leveraging unlabeled data to improve the predictions of a machine learning system is the problem of semi-supervised learning (SSL), which has been the source of many empirical successes \citep{blum1998,kingma2014ssl,dai2017} and theoretical inquiries \citep{azizyan2013,castelli1995,castelli1996,cozman2003,kaariainen2005,niyogi2013,rigollet2007,singh2009,wasserman2008,zhu2003}. 
Commonly studied assumptions include identifiability of the class conditional distributions \citep{castelli1995,castelli1996}, the cluster assumption \citep{rigollet2007,singh2009} and the manifold assumption \citep{zhu2003,wasserman2008,niyogi2013}. In this work, we propose a new type of assumption that loosely combines ideas from both the identifiability and cluster assumption perspectives. Importantly, we consider the general multiclass ($K>2$) scenario, which introduces significant complications. In this setting, we study the sample complexity and rates of convergence for SSL and propose simple algorithms to implement the proposed estimators.

The basic question behind SSL is to connect the marginal distribution over the unlabeled data $\pr(X)$ to the regression function $\pr(Y\given X)$. 
We consider multiclass classification, so that $Y\in\baseY=\{\class_{1},\ldots,\class_{K}\}$ for some $K\ge 2$. In order to motivate our perspective, let $\truedens$ denote the marginal density of the unlabeled samples and suppose that $\truedens$ can be written as a mixture model
\begin{align}
\label{eq:ident:mix}
\truedens(x)
=\sum_{b=1}^{K}\wgt_{b}\denscmp_{b}(x).
\end{align}

\noindent
Crucially, we \emph{do not} assume that each $\denscmp_{b}$ corresponds to some $\truedenscmp_{k}$, where $\truedenscmp_{k}$ is the density of the $k$th class conditional $\pr(X\given Y=\class_{k})$. Nor do we assume that $\wgt_{b}$ corresponds to some $\truewgt_{k}$ where $\truewgt_{k}=\pr(Y=\class_{k})$. We assume that the number of mixture components $K$ is the same as the number of classes.
Assuming the unlabeled data can be used to learn the mixture model \eqref{eq:ident:mix}, the question becomes \emph{when is this mixture model useful for predicting $Y$?} Figure~\ref{fig:mainidea} illustrates an idealized example.

\begin{figure}
\centering
\begin{subfigure}[t]{0.3\textwidth}
\centering
\includegraphics[width=\textwidth]{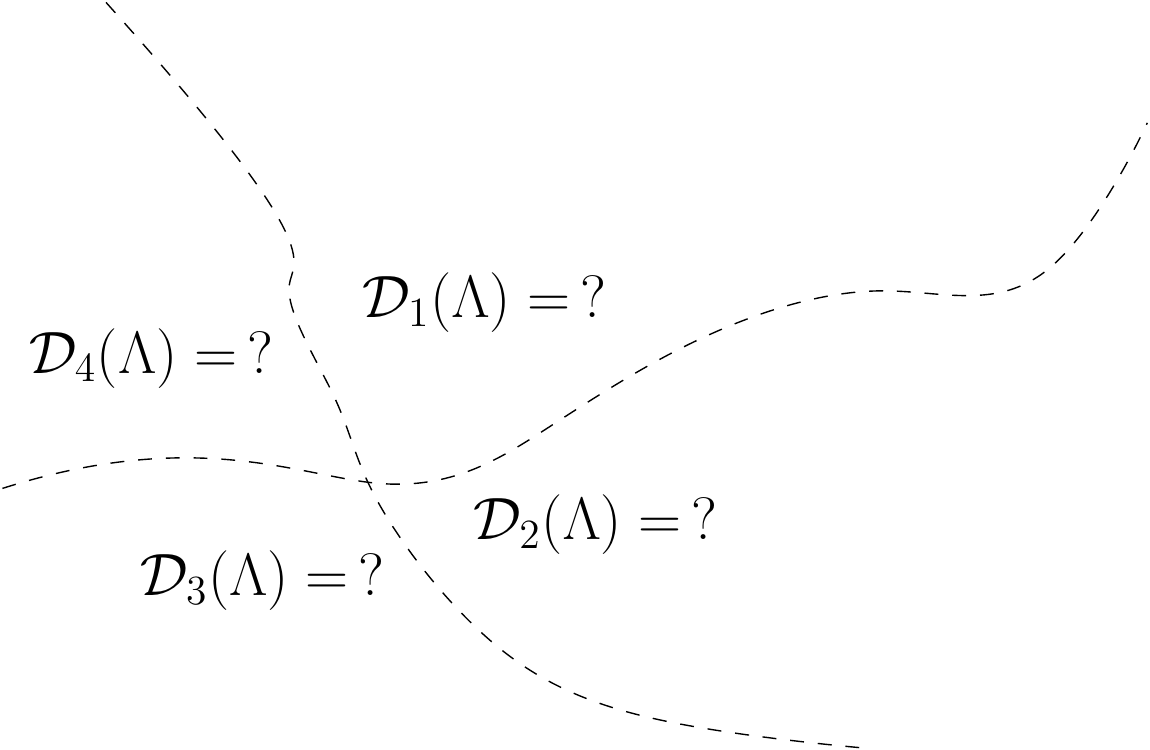} 
\caption{Unknown class assignment.}
\label{fig:mainidea:perm}
\end{subfigure}%
~
\begin{subfigure}[t]{0.3\textwidth}
\centering
\includegraphics[width=\textwidth]{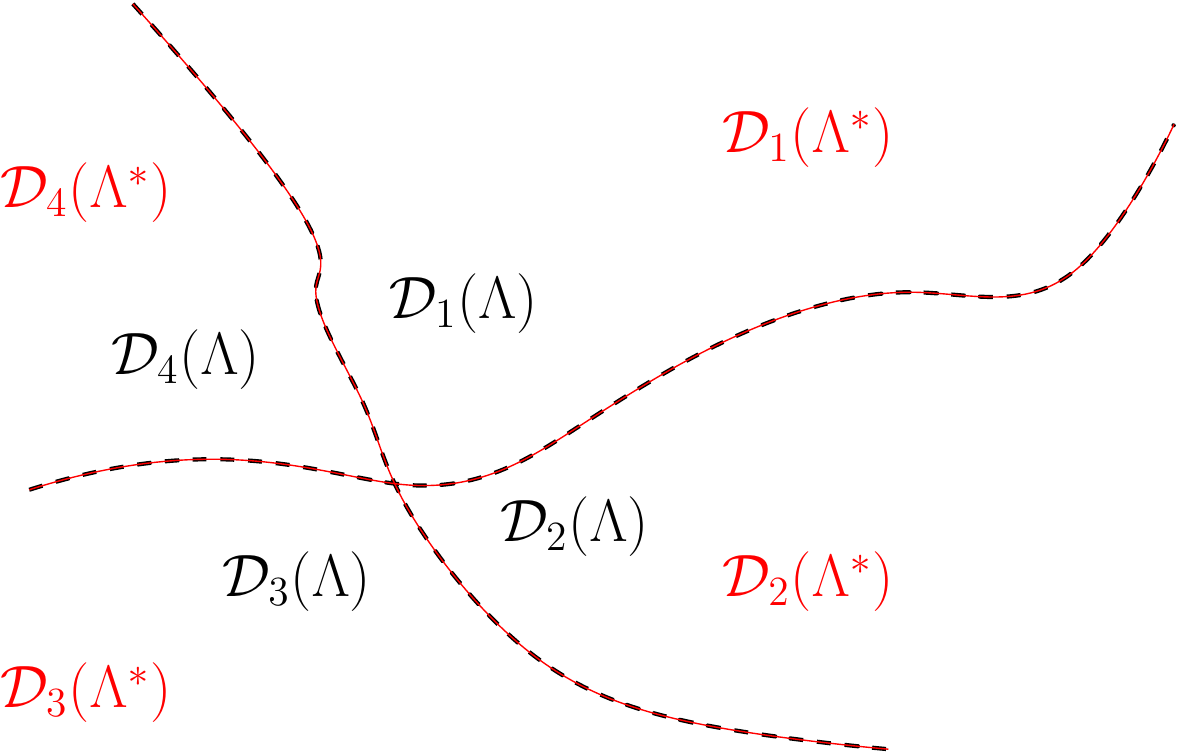}
\caption{True decision boundaries (red) are exactly identified.}
\label{fig:mainidea:prev}
\end{subfigure}%
~
\begin{subfigure}[t]{0.3\textwidth}
\centering
\includegraphics[width=\textwidth]{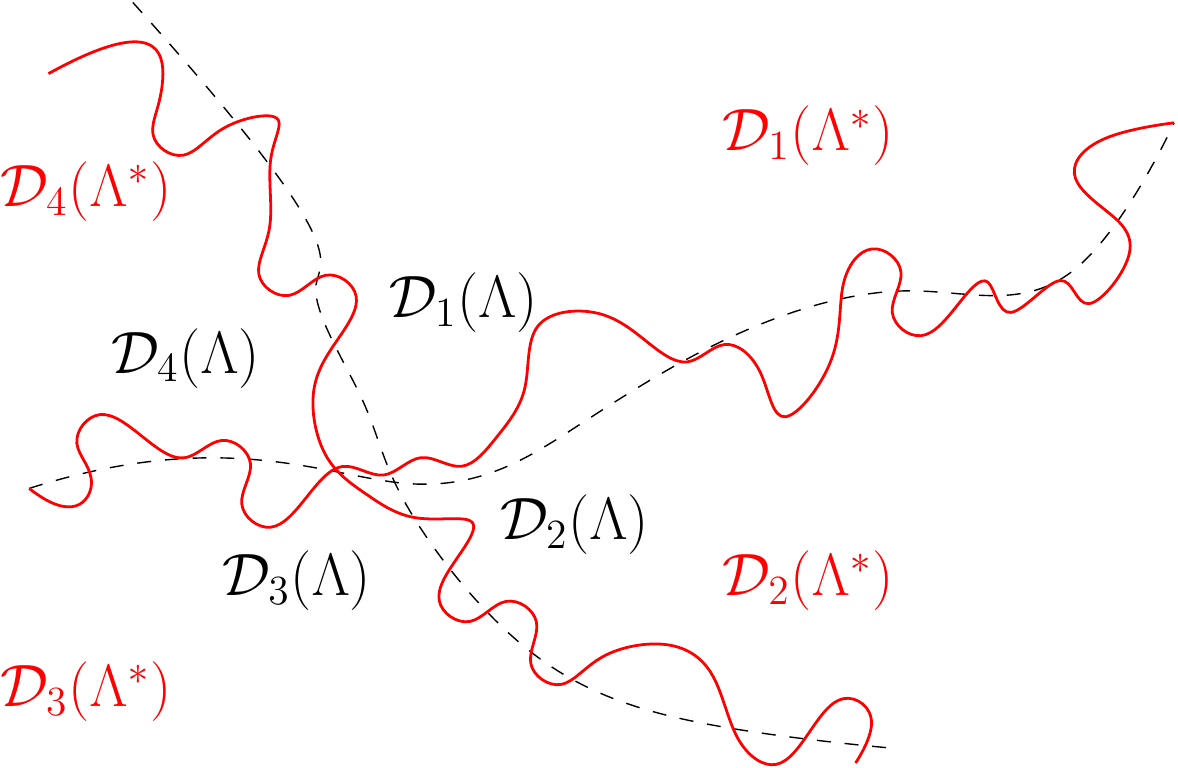}
\caption{True decision boundaries (red) are approximately identified.}
\label{fig:mainidea:current}
\end{subfigure}%
\caption{Illustration of the main idea for $K=4$. The decision boundaries learned from the unlabeled data (cf. \eqref{eq:ident:mix}) are depicted by the dashed black lines and the true decision boundaries are depicted by the solid red lines. 
(a) The unlabeled data is used to learn some approximate decision boundaries through the mixture model $\mix$. Even with the decision boundaries, it is not known which class each region corresponds to. The labeled data is used to learn this assignment.
(b) Previous work assumes that the true and approximate decision boundaries are the same.
(c) In the current work, we assume that the true decision boundaries are unknown, but that it is possible to learn a mixture model that approximates the true boundaries using unlabeled data. 
}
\label{fig:mainidea}
\end{figure}

In an early series of papers, \citet{castelli1995,castelli1996} considered this question under the following assumptions: (a) For each $b$ there is some $k$ such that $\denscmp_{b}=\truedenscmp_{k}$ and $\wgt_{b}=\truewgt_{k}$, (b) $\truedens$ is known, and (c) $K=2$. Thus, they assumed that the true components and weights were known but it was unknown which class each mixture component represents. In Figure~\ref{fig:mainidea}, this corresponds to the case (b) where the decision boundaries are identical. Given labeled data, the special case $K=2$ reduces to a simple hypothesis testing problem which can be tackled using the Neyman-Pearson lemma. 
In this paper, we are interested in settings where each of these three assumptions fail:
\begin{itemize}
\item[(a)] \emph{What if the class conditionals $\truedenscmp_{k}$ are unknown?}
Although we can always write $\truedens(x)=\sum_{k}\truewgt_{k}\truedenscmp_{k}(x)$, it is generally not the case that this mixture model is learnable from unlabeled data alone. In practice, what is learned will be different from this ideal case, but the hope is that it will still be useful. In this case, the argument in \citet{castelli1995} breaks down. Motivated by recent work on nonparametric mixture models \citep{aragam2018npmix}, we study the general case where the true mixture model is not known or even learnable from unlabeled data.
\item[(b)] \emph{What if $\truedens$ is unknown?} In a follow-up paper, \citet{castelli1996} studied the case where $\truedens$ is unknown by assuming that $K=2$ and the class conditional densities $\{\truedenscmp_{1},\truedenscmp_{2}\}$ are known up to a permutation. In this setting, the unlabeled data is used to ascertain the relative mixing proportions, but estimation error in the densities is not considered. We are interested in the general case in which a finite amount of unlabeled data is used to estimate both the mixture weights and densities.
\item[(c)] \emph{What if $K>2$?} If $K>2$, once again the argument in \citet{castelli1995} no longer applies, and we are faced with a challenging permutation learning problem. Permutation learning problems have gained notoriety recently owing to their applicability to a wide variety of problems, including statistical matching and seriation \citep{collier2016,fogel2013,lim2014}, graphical models \citep{geer2013,aragam2016}, and regression \citep{pananjady2016,flammarion2016}, so these results may be of independent interest. 
\end{itemize}
\noindent
With these goals in mind, we study the MLE and majority voting (MV) rules for learning the unknown class assignment introduced in the next section. Our assumptions for MV are closely related to recent work based on the so-called cluster assumption \citep{seeger2000,singh2009,rigollet2007,azizyan2013}; see Section~\ref{sec:theory:vote} for more details.

\paragraph{Contributions} A key aspect of our analysis is to establish conditions that connect the mixture model \eqref{eq:ident:mix} to the true mixture model. Under these conditions we prove nonasymptotic rates of convergence for learning the class assignment (Figure~\ref{fig:mainidea:perm}) from labeled data when $K>2$, establish an $\Omega(K\log K)$ sample complexity for learning this assignment, and prove that the resulting classifier converges to the Bayes classifier. We then propose simple algorithms based on a connection to bipartite graph matching, and illustrate their performance on real and simulated data.

\section{SSL as permutation learning}
\label{sec:perm}

In this section, we formalize the ideas from the introduction using the language of mixing measures. We adopt this language for several reasons: 1) It makes it easy to refer to the parameters in the mixture model \eqref{eq:ident:mix} by wrapping everything into a single, coherent statistical parameter $\mix$, 2) We can talk about convergence of these parameters via the Wasserstein metric, and 3) It simplifies discussions of identifiability in mixture models. Before going into technical details, we summarize the main idea as follows (see also Figure~\ref{fig:mainidea}):
\begin{enumerate}
\item Use the unlabeled data to learn a $K$-component mixture model that approximates $\truedens$, which is represented by the mixing measure $\mix$ defined below;
\item Use the labeled data to determine the correct assignment $\perm$ of classes $\class_{k}$ to the decision regions $\dec_{b}(\mix)$ defined by $\mix$;
\item Based on the pair $(\mix,\perm)$, define a classifier $\clf_{\mix,\perm}:\baseX\to\baseY$ by \eqref{eq:def:classifier} below.
\end{enumerate}

\paragraph{Mixing measures and mixture models}
For concreteness, we will work on $\baseX=\R^{d}$, however, our results generalize naturally to any space $\baseX$ with a dominating measure and well-defined density functions. Let $\alldens=\{f\in L^{1}(\R^{d}) : \int f\,dx=1\}$ be the set of probability density functions on $\R^{d}$, and $\mixing_{K}(\alldens)$ denote the space of probability measures over $\alldens$ with precisely $K$ atoms. An element $\mix\in\mixing_{K}(\alldens)$ is called a \emph{(finite) mixing measure}, and can be thought of as a convenient mathematical device for encoding the weights $\{\wgt_{k}\}$ and the densities $\{\denscmp_{k}\}$ into a single statistical parameter.
By integrating against this measure, we obtain a new probability density which is denoted by
\begin{align}
m(\mix)
:= \sum_{b=1}^{K}\wgt_{b}\denscmp_{b}(x),
\end{align}
\noindent
where $\denscmp_{b}$ is a particular enumeration of the densities in the support of $\mix$ and $\wgt_{b}$ is the probability of the $b$th density. Thus, \eqref{eq:ident:mix} can be written as $\truedens=m(\mix)$. By metrizing $\alldens$ via the total variation distance $\dTV(f,g)=\tfrac12\int|f-g|\,dx$, the distance between two finite $K$-mixtures can be computed via the Wasserstein metric:
\begin{align*}
\begin{aligned}
\wass(\mix, \mix')
= \inf\Bigg\{
\sum_{i,j}\coupling_{ij}\dTV(\denscmp_{i},\denscmp_{j}')
: 0\le \coupling_{ij}\le 1,\,
\sum_{i,j}\coupling_{ij}=1,\,
\sum_{i}\coupling_{ij}=\wgt_{j}',\,
\sum_{j}\coupling_{ij}=\wgt_{i}
\Bigg\}.
\end{aligned}
\end{align*}

\paragraph{Decision regions, assignments, and classifiers} Any mixing measure $\mix$ defines $K$ decision regions given by $\dec_{b}=\dec_{b}(\mix):=\{x\in\baseX : \wgt_{b}\denscmp_{b}(x) > \wgt_{j}\denscmp_{j}(x)\,\forall j\ne b\}$ (Figure~\ref{fig:mainidea}). This allows us to assign an index from $1,\ldots,K$ to any $x\in\baseX$, and hence defines a classifier $\uclf_{\mix}:\baseX\to[K]:=\{1,\ldots,K\}$. This classifier does not solve the original labeled problem, however, since the output is an uninformative index $b\in[K]$ as opposed to a proper class label $\class_{k}\in\baseY$. The key point is that even if we know $\mix$, we still must identify each label $\class_{k}$ with a decision region $\dec_{b}(\mix)$, i.e. we must learn a permutation $\perm:\baseY\to[K]$. With some abuse of notation, we will sometimes write $\perm(k)$ instead of $\perm(\class_{k})$ for any permutation $\perm$. Together a pair $(\mix,\perm)$ defines a classifier $\clf_{\mix,\perm}:\baseX\to\baseY$ by
\begin{align}
\label{eq:def:classifier}
\clf_{\mix,\perm}(x) 
= \perm(\uclf_{\mix}(x))
= \sum_{b=1}^{K}\perm^{-1}(b)1(x\in\dec_{b}(\mix)).
\end{align}

\noindent
This mixing measure perspective helps to clarify the role of the unknown permutation in supervised learning: The unlabeled data is enough to learn $\mix$ (and hence the decision regions $\dec_{b}(\mix)$), however, labeled data are necessary to learn an assignment $\perm$ between classes and decision regions.

This formulates SSL as a coupled mixture modeling and permutation learning problem: Given unlabeled and labeled data, learn a pair $(\estmix,\estperm)$ which yields a classifier $\estclf=\clf_{\estmix,\estperm}$. The target is the \emph{Bayes classifier}, which can also be written in the form \eqref{eq:def:classifier}: Let $\truemix$ denote the mixing measure that assigns probability $\truewgt_{k}$ to the density $\truedenscmp_{k}$ and note that $\truedens=m(\truemix)$, which is the \emph{true mixture model} defined previously. Let $\trueperm:\baseY\to[K]$ be the permutation that assigns each class $\class_{k}$ to the correct decision region $\dec_{b}^{*}=\dec_{b}(\truemix)$ (Figure~\ref{fig:mainidea}). Then it is easy to check that $\clf_{\truemix,\trueperm}$ is the Bayes classifier. 

\paragraph{Identifiability} Although the true mixing measure $\truemix$ may not be identifiable from $\truedens$, some other mixture model may be. In other words, although it may not be possible to learn $\truemix$ from unlabeled data, it may be possible to learn some other mixing measure $\mix\ne\truemix$ such that $m(\mix)=\truedens=m(\truemix)$ (Figure~\ref{fig:mainidea:current}). This essentially amounts to a violation of the cluster assumption: High-density clusters are identifiable, but in practice the true class labels may not respect the cluster boundaries. Assumptions that guarantee a mixture model are identifiable are well-studied \citep{teicher1961,teicher1963,yakowitz1968}, including both parametric \cite{barndorff1965} and nonparametric \citep{aragam2018npmix,teicher1967,hall2003} assumptions. In particular, \citet{aragam2018npmix} have proved general conditions under which mixture models with arbitrary, overlapping nonparametric components are identifiable and estimable, including examples where each component $\denscmp_{k}$ has the same mean. Since this problem is well-studied, we focus hereafter on the problem of learning the permutation $\trueperm$. Thus, in the sequel we will assume that we are given an arbitrary mixing measure $\mix$ which will be used to estimate $\trueperm$. We do not assume that $\mix=\truemix$ or even that these mixing measures are close. The idea is to elicit conditions on $\mix$ that ensure consistent estimation of $\trueperm$.

\section{Two estimators}
\label{sec:est}

Assume we are given a mixing measure $\mix$ along with the labeled samples $(\Xi,\Yi)\in\baseX\times\baseY$. 
Two natural estimators of $\trueperm$ are the MLE and majority vote. Although both estimators depend on $\mix$, this dependence will be suppressed for brevity.

\paragraph{Maximum likelihood} Define $\nll(\perm;\mix,X,Y):=\log\wgt_{\perm(Y)}\denscmp_{\perm(Y)}(X)$. We will work with the following \emph{misspecified} MLE (i.e. $\mix\ne\truemix$)
\begin{align}
\label{eq:defn:mle}
\estmle
\in\argmax_{\perm}\nll_{n}(\perm;\,\mix),
\quad
\nll_{n}(\perm;\,\mix)
:=\frac1n\sum_{i=1}^{n}\nll(\perm;\mix,\Xi,\Yi).
\end{align}

\noindent
When $\mix=\truemix$, this is the correctly specified MLE of the unknown permutation $\trueperm$, however, the definition above allows for the general misspecified case $\mix\ne\truemix$.

\paragraph{Majority vote} The majority vote estimator (MV) is given by a simple majority vote over each decision region. Formally, we define a permutation $\estmv$ as follows: The inverse assignment $\estmv^{-1}:[K]\to\baseY$ is defined by
\begin{align}
\label{eq:defn:vote}
\estmv^{-1}(b)
= \argmax_{\class\in\baseY} \sum_{i=1}^{n}\ind(\Yi=\class, \Xi\in\dec_{b}(\mix)).
\end{align}

\noindent
If there is no majority class in a given decision region, we consider this a failure of MV and treat it as undefined. 
Note that when $K=2$, the MV classifier defined by \eqref{eq:def:classifier} with $\perm=\estmv$ is essentially the same as the three-step procedure described in \citet{rigollet2007}, which focuses on bounding the excess risk under the cluster assumption. In contrast, we are interested in the consistency of the unknown permutation $\trueperm$ when $K>2$, which is a more difficult problem.

\section{Statistical results}
\label{sec:theory}

Our main results establish rates of convergence for both the MLE and MV introduced in the previous section. We will use the notation $\E_{*} h(X,Y)$ to denote the expectation with respect to the true distribution $(X,Y)\sim\pr(X,Y)$. Without loss of generality, we assume that $\trueperm(\class_{k})=k$ and $\denscmp_{b}=\truedenscmp_{b}+h_{b}$ for some $h_{b}$. Then $\estperm=\trueperm$ if and only if $\estperm(\alpha_{k})=k$, which helps to simplify the notation in the sequel.

\subsection{Maximum likelihood}
\label{sec:theory:mle}

Given $\mix$, the notation $\E_{*}\nll(\perm;\mix,X,Y)=\E_{*}\log\wgt_{\perm(Y)}\denscmp_{\perm(Y)}(X)$ denotes the expectation of the \emph{misspecified} log-likelihood with respect to the \emph{true} distribution. 
Define the ``gap''
\begin{align}
\gapmle(\mix) 
:= \E_{*}\nll(\trueperm;\mix,X,Y) - \max_{\perm\ne\trueperm}\E_{*}\nll(\perm;\mix,X,Y).
\end{align}
\noindent
For any function $a:\R\to\R$, define the usual Fenchel-Legendre dual $a^{*}(t)=\sup_{s\in\R}(s t - a(s))$. Let $U_{b}=\log\wgt_{b}\denscmp_{b}(X)$ and $\beta_{b}(s)=\log\E_{*}\exp(s U_{b})$. 
Finally, let $n_{k}:=|\{i:\Yi=\class_{k}\}|$ denote the number of labeled samples with the $k$th label.

\begin{thm}
\label{thm:main:mle}
Let $\estmle$ be the MLE defined in \eqref{eq:defn:mle}.
If $\gapmle:=\gapmle(\mix)>0$ then 
\begin{align*}
\pr(\estmle=\trueperm) 
& \geq 1 - 2 K^2\exp\Big(-\inf_k n_k \cdot \inf_b \beta_{b}^* (\gapmle/3)\Big).
\end{align*}
\end{thm}

The condition $\gapmle(\mix)>0$ is a crucial condition that ensures that $\trueperm$ is learnable from $\mix$, and the size of $\gapmle(\mix)$ quantifies ``how easy'' it is to learn $\trueperm$ is given $\mix$. A bigger gap implies an easier problem. Thus, it is of interest to understand this quantity better. The following proposition shows that when $\mix=\truemix$, this gap is always nonnegative:

\begin{prop}
    \label{prop:gap:mle}
    For any permutation $\perm$ and any $\mix$, 
    \begin{align*}
    \E_{*}\nll(\perm;\mix,X,Y) \leq \E_{*}\nll(\trueperm;\truemix,X,Y)
    \end{align*}

    \noindent
    and hence $\gapmle(\truemix)\ge 0$.
\end{prop}

\noindent
In general, assuming $\gapmle(\mix)>0$ is a weak assumption, but bounds on $\gapmle(\mix)$ are difficult to obtain without making additional assumptions on the densities $\denscmp_{k}$ and $\truedenscmp_{k}$. A brief discussion of this can be found in Appendix~\ref{app:conditions}; we leave it to future work to study this quantity more carefully.

\subsection{Majority vote}
\label{sec:theory:vote}

For any $\mix$, define $m_b := |i:\Xi\in\dec_{b}(\mix)|$ and $\chi_{bj}(\mix):=\frac{1}{m_b}\sum_{i=1}^{n}\ind(\Yi=j, \Xi\in\dec_{b}(\mix))$, where $\ind(\cdot)$ is the indicator function. Similar to the MLE, our results for MV depend crucially on a ``gap'' quantity, given by
\begin{align}
\label{eq:gap:vote}
\gapmv(\mix)
:= \inf_{b}\Big\{\E_{*}\chi_{bb}(\mix) - \max_{j\ne b}\E_{*}\chi_{bj}(\mix)\Big\}.
\end{align}
\noindent
This quantity essentially measures how much more likely it is to sample the $b$th label in the $b$th decision region than any other label, averaged over the entire region. Thus, conditions on $\gapmv(\mix)$ are closely related to the well-known cluster assumption \citep{seeger2000,singh2009,rigollet2007,azizyan2013}.

\begin{thm}
\label{thm:main:vote}
Let $\estmv$ be the MV defined in \eqref{eq:defn:vote}.
If $\gapmv:=\gapmv(\mix)>0$ then 
\begin{align*}
\pr(\estmv=\trueperm) 
& \geq 1 - 2 K^2 \exp\Big( \frac{-2\gapmv^2\min_b{m_b}}{9}\Big).
\end{align*}
\end{thm}

\noindent
As with the MLE, the gap $\gapmv(\mix)$ is a crucial quantity. Fortunately, when $\mix=\truemix$ it is always positive:

\begin{prop}
    \label{prop:gap:vote}
    For each $b=1,\ldots,K$, 
    \begin{align*}
    \E_{*}\chi_{bb}(\truemix) > \max_{j\ne b}\E_{*}\chi_{bj}(\truemix)
    \end{align*}
    \noindent
    and hence $\gapmv(\truemix)> 0$.
\end{prop}

\noindent
When $\mix\ne\truemix$, $\gapmv(\mix)$ has the following interpretation: $\gapmv(\mix)$ measures how well the decision regions defined by $\mix$ match up with the decision regions defined by $\truemix$. When $\mix$ defines decision regions that assign high probability to one class, $\gapmv(\mix)$ will be large. If $\mix$ defines decision regions where multiple classes have approximately the same probability, however, then it is possible that $\gapmv(\mix)$ will be small. In this case, our experiments in Section~\ref{sec:exp} indicate that the MLE performs much better by managing overlapping decision regions more gracefully.

\subsection{Sample complexity}
\label{sec:theory:samples}

Theorems~\ref{thm:main:mle} and~\ref{thm:main:vote} imply upper bounds on the minimum number of samples required to learn the permutation $\trueperm$: For any $\delta\in(0,1)$, as long as
\begin{align}
\label{eq:samp:mle}
\textup{(MLE)}
\qquad
\inf_{k} n_k 
:= n_{0} 
&\geq \frac{\log\frac{2K^2}{\delta}}{\inf_b \beta_{b}^* (\gapmle/3)} \\
\label{eq:samp:vote}
\textup{(MV)}\,
\qquad
\inf_{b} m_b
:= m_{0}
&\geq \frac{9\log \frac{2K^2}{\delta}}{2\gapmv^{2}} 
\end{align}

\noindent
we recover $\trueperm$ with probability at least $1-\delta$.

To derive the sample complexity in terms of the total number of labeled samples $n$, it suffices to determine the minimum number of samples per class given $n$ draws from a multinomial random variable. For the general case with unequal probabilities, Lemma~\ref{lem:multinomial:min} provides a precise answer. For simplicity here, we summarize the special case where each class (resp. decision region) is equally probable for the MLE (resp. MV).

\begin{cor}[Sample complexity of MLE]
\label{cor:sample:mle}
Suppose that $\truewgt_{k}=1/K$ for each $k$, $\gapmle>0$, and
\begin{align*}
n
\ge K\log(K/\delta)\Big[1 + \frac{4}{\inf_b \beta_{b}^* (\gapmle/3)} \Big].
\end{align*}

\noindent 
Then $\pr(\estmle=\trueperm)\ge 1-\delta$.
\end{cor}

\begin{cor}[Sample complexity of MV]
\label{cor:sample:vote}
Suppose that $\pr(X\in\dec_{b}(\mix))=1/K$ for each $k$, $\gapmv>0$, and
\begin{align*}
n
\ge K\log(K/\delta)\Big[1 + \frac{18}{\gapmv^{2}} \Big].
\end{align*}

\noindent 
Then $\pr(\estmv=\trueperm)\ge 1-\delta$.
\end{cor}

\paragraph{Coupon collector's problem and SSL} To better understand these bounds, consider arguably the simplest possible case: Suppose that each density $\truedenscmp_{k}$ has disjoint support, $\truewgt_{k}=1/K$, and that we know $\truemix$. Under these very strong assumptions, an alternative way to learn $\trueperm$ is to simply sample from $\pr(X)$ until we have visited each decision region $\truedec_{k}$ at least once. This is the classical \emph{coupon collector's problem} (CCP), which is known to require $\Theta(K\log K)$ samples \citep{newman1960,flajolet1992}. Thus, under these assumptions the expected number of samples required to learn $\trueperm$ is $\Theta(K\log K)$. 
By comparison, our results indicate that \emph{even if the $\truedenscmp_{k}$ have overlapping supports} and \emph{we do not know $\truemix$}, as long as $\gapmle=\Omega(1)$ (resp. $\gapmv=\Omega(1)$) then $\Omega(K\log K)$ samples suffice to learn $\trueperm$. In other words, SSL is approximately as difficult as CCP in very general settings.

\subsection{Classification error}
\label{sec:theory:error}

So far our results have focused on the probability of recovery of the unknown permutation $\trueperm$. 
In this section, we bound the classification error of the classifier \eqref{eq:def:classifier} in terms of the Wasserstein distance $\wass(\mix,\truemix)$ between $\mix$ and $\truemix$. We assume the following general set-up: We are given $m$ unlabeled samples from which we estimate $\mix$ by $\estmix_{m}$. Based on this mixing measure, we learn a permutation $\estperm_{m,n}$ from $n$ labeled samples, e.g. using either MLE \eqref{eq:defn:mle} or MV \eqref{eq:defn:vote}.  Together, the pair $(\estmix_{m},\estperm_{m,n})$ defines a classifier $\estclf_{m,n}$ via \eqref{eq:def:classifier}. We are interested in bounding the probability of misclassification $\pr(\estclf_{m,n}(X)\ne Y)$ in terms of the Bayes error.

\begin{thm}[Classification error]
\label{thm:main:error}
Suppose $\wass(\estmix_{m},\mix)=O(r_{m})$ for some $r_{m}\to0$ where $m$ is the number of unlabeled samples. Let $\trueclf=\clf_{\truemix,\trueperm}$ denote the Bayes classifier. Then there is a constant $C>0$ depending on $K$ and $\truemix$ such that if $\estperm_{m,n}=\trueperm$, 
\begin{align*}
\pr(\estclf_{m,n}(X)\ne Y)
\le \pr(\trueclf(X)\ne Y) + Cr_{m} + C\cdot \wass(\mix,\truemix).
\end{align*}
\end{thm}

This theorem allows for the possibility that the mixture model learned from the unlabeled data (i.e. $\estmix_{m}$) does not converge to the true mixing measure $\truemix$. In this case, there is an irreducible error quantified by the Wasserstein distance $\wass(\mix,\truemix)$. When $\wass(\mix,\truemix)=0$, however, we can improve this upper bound considerably to yield nonasymptotic rates of convergence to the Bayes error rate:
\begin{cor}
\label{cor:bayes:error}
If $\wass(\estmix_{m},\truemix)=O(r_{m})$ for some $r_{m}\to0$, then the excess risk of $\estclf_{m,n}$ converges to zero at the same rate as $\wass(\estmix_{m},\truemix)$:
\begin{align*}
\pr(\estclf_{m,n}(X)\ne Y) - \pr(\trueclf(X)\ne Y) 
= O(r_{m}).
\end{align*}
\end{cor}

\paragraph{Clairvoyant SSL}
Previous work \citep{castelli1995,castelli1996,singh2009} has studied the so-called \emph{clairvoyant} SSL case in which it is assumed that we know \eqref{eq:ident:mix} perfectly. This amounts to taking $\estmix_{m}=\mix$ in the previous results, or equivalently $m=\infty$. Under this assumption, we have perfect knowledge of the decision regions and only need to learn the label permutation $\trueperm$. Then Corollary~\ref{cor:bayes:error} implies that with high probability, we can learn a Bayes classifier for the problem using finitely many labeled samples.

\paragraph{Convergence rates}
The convergence rate $r_{m}$ used here is essentially the rate of convergence in estimating an identifiable mixture model, which is well-studied for parametric mixture models \citep{heinrich2015,ho2016,ho2016singularity}. In particular, for so-called \emph{strongly} identifiable parametric mixture models, the minimax rate of convergence attains the optimal root-$m$ rate $r_{m}=m^{-1/2}$ \citep{heinrich2015}.\footnote{This paper corrects an earlier result due to \citet{chen1995} that claimed an $m^{-1/4}$ minimax rate.} Asymptotic consistency theorems for nonparametric mixtures can be found in \citet{aragam2018npmix}.

\paragraph{Comparison to supervised learning (SL).}
Previous work \citep{singh2009} has compared the sample complexity of SSL to SL under a cluster-type assumption. While a precise characterization of these trade-offs is not the main focus of this paper, we note in passing here the following: If the minimax risk of SL for a particular problem is larger than $\wass(\mix,\truemix)$, then Theorem~\ref{thm:main:error} implies that SSL provably outperforms SL on finite samples.

\subsection{Discussion of conditions}
\label{app:conditions}

Here we have a simple experiment with the underlying distribution being a mixture of two Gaussians:
\[F = \frac{1}{2} \lambda_1^* + \frac{1}{2} \lambda_2^* = \frac{1}{2} \normalN(-\mu,1) + \frac{1}{2} \normalN(\mu,1)\]
where $\mu$ is a small positive number indicating the separation between two Gaussians. We would like to compare the number of samples needed to recover the true permutation $\trueperm$ with probability $(1-\delta)$ for both MLE and MV. 

Our experiments show that both estimators have roughly $O(\mu^{-2})$ sample complexity when $\mu \rightarrow 0^+$, but MV needs about \textbf{4 times} as many samples as the MLE. In fact, our theory can verify the sample complexity of MV: The gap $\gapmv$ is $\Phi(\mu)- \Phi(-\mu) = O(\mu) $ and the sample complexity has $\log(K/\delta)/\gapmv^2$ dependence with $\gapmv$, which gives exactly $O(\mu^{-2})$. Here $\Phi(\mu)$ is the cumulative distribution function of standard normal random variable. Unfortunately, the intractable form of the dual functions $\beta_{b}^{*}$ makes similar analytical comparisons difficult.

\section{Algorithms}
\label{sec:alg}

One of the significant appeals of MV \eqref{eq:defn:vote} is its simplicity. It is conceptually easy to understand and trivial to implement. The MLE \eqref{eq:defn:mle}, on the other hand, is more subtle and difficult to compute in practice. In this section, we discuss two algorithms for computing the MLE: 1) An exact algorithm based on finding the maximum weight perfect matching in a bipartite graph by the Hungarian algorithm \citep{kuhn1955}, and 2) Greedy optimization.

Define $C_k = \{i : \Yi=\class_{k}\}$. Consider the weighted complete bipartite graph $G=(V_{K,K},w)$ with edge weights
\begin{align*}
w(k, k') 
= \sum_{i \in C_k } \log \big( \wgt_{k'} \denscmp_{k'}(\Xi) \big), 
\quad \forall k,k' \in [K] 
\end{align*}
\noindent
Since a permutation $\perm$ defines a perfect matching on $G$, the log-likelihood can be rewritten as 
\begin{align*}
\nll_{n}(\perm;\mix) 
= \sum_{k=1}^K \sum_{i\in C_{k}}\log \big( \lambda_{\perm(\class_{k}) } \denscmp_{\perm(\class_{k}) }(\Xi)\big)
= \sum_{k=1}^K  w(k, \perm(\class_{k})),
\end{align*}
\noindent
the right side of which is the total weight of the matching $\perm$. Hence, the maximizer $\estmle$ can be found by finding a perfect matching for this graph that has maximum weight. This can be done in $O(K^{3})$ using the well-known Hungarian algorithm \citep{kuhn1955}.

We can also approximately solve the matching problem by a greedy method: Assign the $k$th class to
\begin{align*}
\estgreedy(\class_{k}) 
= \argmax_{k'\in[K]} w(k,k') = \argmax_{k'\in[K]} \sum_{i \in C_k } \log \big( \wgt_{k'} \denscmp_{k'}(\Xi) \big),
\end{align*}
This greedy heuristic isn't guaranteed to achieve optimal matching, however, it is simple to implement and can be viewed as a ``soft interpolation'' of $\estmle$ and $\estmv$ as follows:
If we define $w_{\rm{MV}}(k,k') = \sum_{i \in C_k } \ind(\Xi\in\dec_{k'}(\mix)) $, we can see that a training example $(\Xi, \Yi=\class_{k})$ contributes $1$ to $w_{\rm{MV}}(k,k')$ if $k' = \argmax_{j}\wgt_{j} \denscmp_{j}(\Xi)$, and contributes $0$ to $w_{\rm{MV}}(k,k')$ otherwise. By comparison, for the greedy heuristic, a training example $(\Xi, \Yi=\class_{k})$ contributes $\log (\wgt_{k'} \denscmp_{k'}(\Xi)) $ to $w(k,k')$. Therefore, the greedy estimator can be seen as a ``soft'' version of MV that also greedily optimizes the MLE objective.

\section{Experiments}
\label{sec:exp}

In order to evaluate the performance of the proposed estimators in practice, we implemented each of the three methods described in Section~\ref{sec:alg} on simulated and real data. Our experiments considered three settings: (i) Parametric mixtures of Gaussians, (ii) A nonparametric mixture model, and (iii) Real data from MNIST.
All three experiments followed the same pattern: A random mixture model $\truemix$ was generated, and then $N=99$ labeled samples were drawn from this mixture model. We generated $\truemix$ under different separation conditions, from well-separated to overlapping. Then, $\mix$ was generated in two ways: (a) $\mix=\truemix$, corresponding to a setting where the true decision boundaries are known, and (b) $\mix\ne\truemix$ by perturbing the components and weights of $\truemix$ by a parameter $\eta>0$ (see below for details). Then $\mix$ was used to estimate $\trueperm$ using each of the three algorithms described in the previous section for the first $n=3,6,9,\ldots,99$ labeled samples. This procedure was repeated $T=50$ times (holding $\truemix$ and $\mix$ fixed) in order to estimate $\pr(\estperm=\trueperm)$. Figure~\ref{fig:k16} depicts some examples of the mixtures used in our experiments.

\begin{figure}[t]
\includegraphics[width=0.5\textwidth]{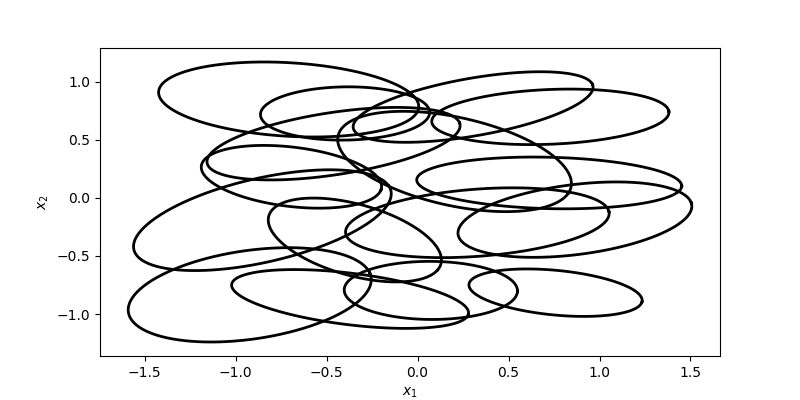}
\includegraphics[width=0.5\textwidth]{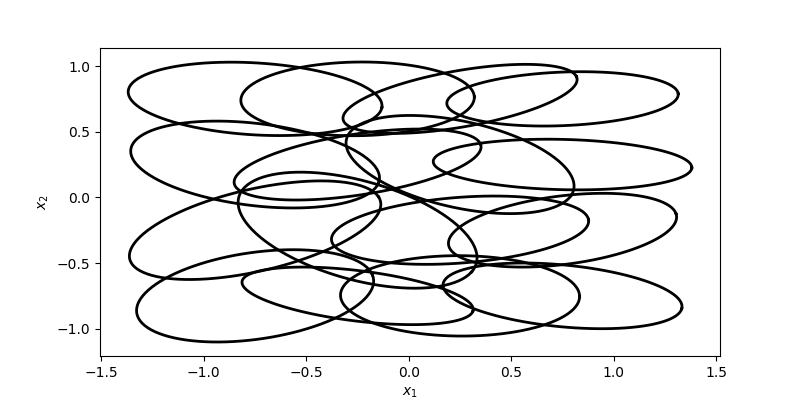}\\
\includegraphics[width=0.5\textwidth]{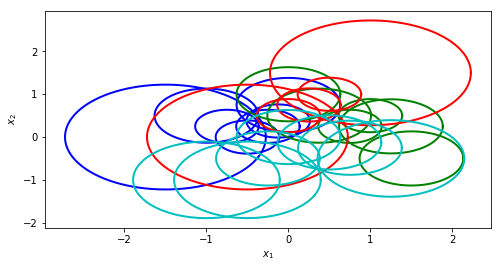}
\includegraphics[width=0.5\textwidth]{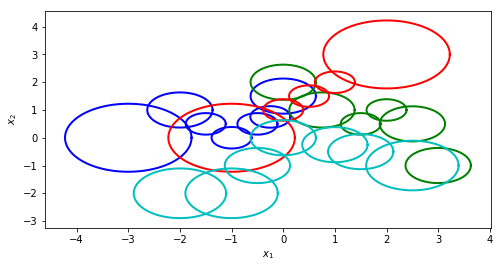}
\caption{Examples of some examples used in the experiments. Depicted are contour lines of the densities for one standard deviation from the mean. (top) Mixture of Gaussians with $K=16$. (bottom) Nonparametric mixture of Gaussian mixtures; each Gaussian component is coloured according to the class label it generates.}
\label{fig:k16}
\end{figure}




\paragraph{Mixture of Gaussians}  The first experiment uses synthetic data where $F = \sum_k \lambda_k^* f_k^*$ is a mixture of Gaussians with $\lambda_k^*$ being randomly drawn from a uniform distribution $\mathcal{U}(0,1)$ (normalized afterwards) and $f_k^*$ being a Gaussian density. The $f_k^*$ were arranged on a square grid with randomly generated positive-definite covariance matrices.

To explicitly control how well-separated the Gaussians are, we shrink the expectations of the Gaussians towards the origin using a parameter $\eta$ where $\eta\in\{1, 0.75, 0.5\}$. We design the means of the Gaussians so that they are on a grid centered at the origin. The mean of each Gaussian component is thus given by $\eta\mu_{k}^{*}$, where $\mu_{k}^{*}$ is the mean of the $k$th density. When $\eta = 1$, components in the mixture are well-separated where $\{f_k^*\}_{k=1}^K$ have no or very little overlap within one standard deviation. The smaller the $\eta$ is, the more overlapping the components are. For each choice of dimension $d\in\{2,10\}$, $K$ is varied between $\{2, 4 ,9, 16\}$. 

\paragraph{Perturbed mixture of Gaussians} In this setting, we test the case where $\Lambda^*$ is unknown and the algorithms only have access to its perturbed version $\Lambda$. Similar to the above setups, we sample $n$ labeled data using $\Lambda^*$. However, instead of feeding the algorithms the true mixture $\Lambda^*$, we input $\Lambda$ where mixture weights are shifted: Each dimension of the means of the Gaussians are shifted by a random number drawn from $\mathcal{N}(0,0.1)$ and the variance of each Gaussians is scaled by either $0.5$ or $2$ (chosen at random).

\paragraph{Mixture of Gaussian mixtures and its perturbation} 
This experiment is similar to the first experiment with a mixture of Gaussians except each $f_k^*$ is itself a Gaussian mixture. We also controlled the degree of separation by shrinking the expectation of each Gaussian towards the origin with $\eta\in\{1,0.5\}$.  

\paragraph{MNIST and corrupted MNIST}  We trained $10$ kernel density estimators (one for each digit) for $\{f_k\}_{k=1}^{10}$. These mixtures are used to define the true mixture $\Lambda^{*}$. We then tested, under corruption of the labeled samples from the test set, how the three algorithms behave. With probability $0.1$, the label of the sampled data is changed to an incorrect label. 

\bigskip
The results are depicted in Figure~\ref{fig:mnist}. As expected, the MLE performs by far the best, obtaining near perfect recovery of $\trueperm$ with fewer than $n=20$ labeled samples on synthetic data, and fewer than $n=40$ on MNIST. Unsurprisingly, the most difficult case was $K=16$, in which only the MLE was able recover the true permutation $>50\%$ of the time. By increasing $n$, the MLE is eventually able to learn this most difficult case, in accordance with our theory. Furthermore, the MLE is much more robust to misspecification $\mix\ne\truemix$ and component overlap compared to the others. This highlights the advantage of leveraging density information in the MLE, which is ignored by the MV estimator (i.e. MV only uses decision regions).

\begin{figure}[hbt]
\begin{center}
\begin{subfigure}{\textwidth}
  \centering
  \includegraphics[width=0.8\linewidth]{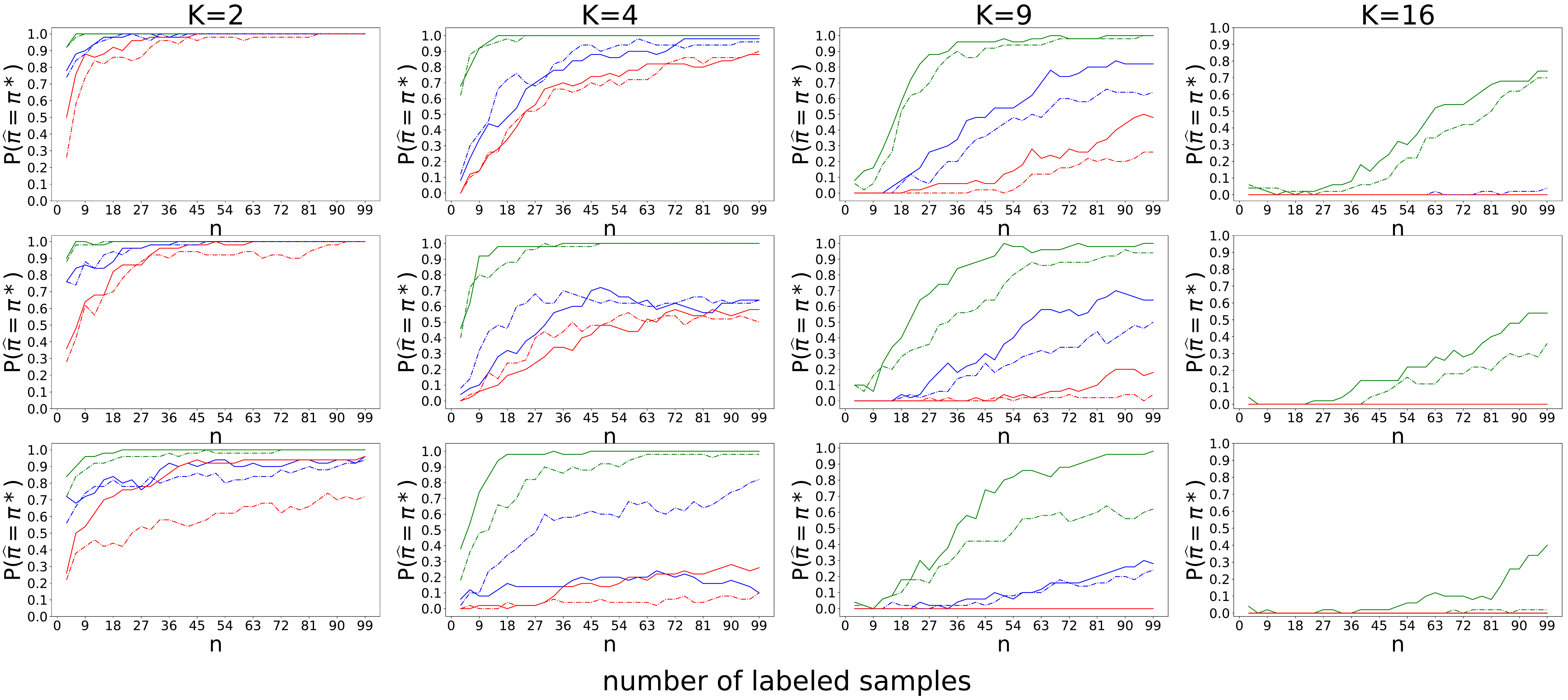}
  \caption{Mixture of Gaussians}
  \label{subfig:momog}
\end{subfigure}\\
\begin{subfigure}{.44\textwidth}
  \centering
  \includegraphics[width=0.8\linewidth]{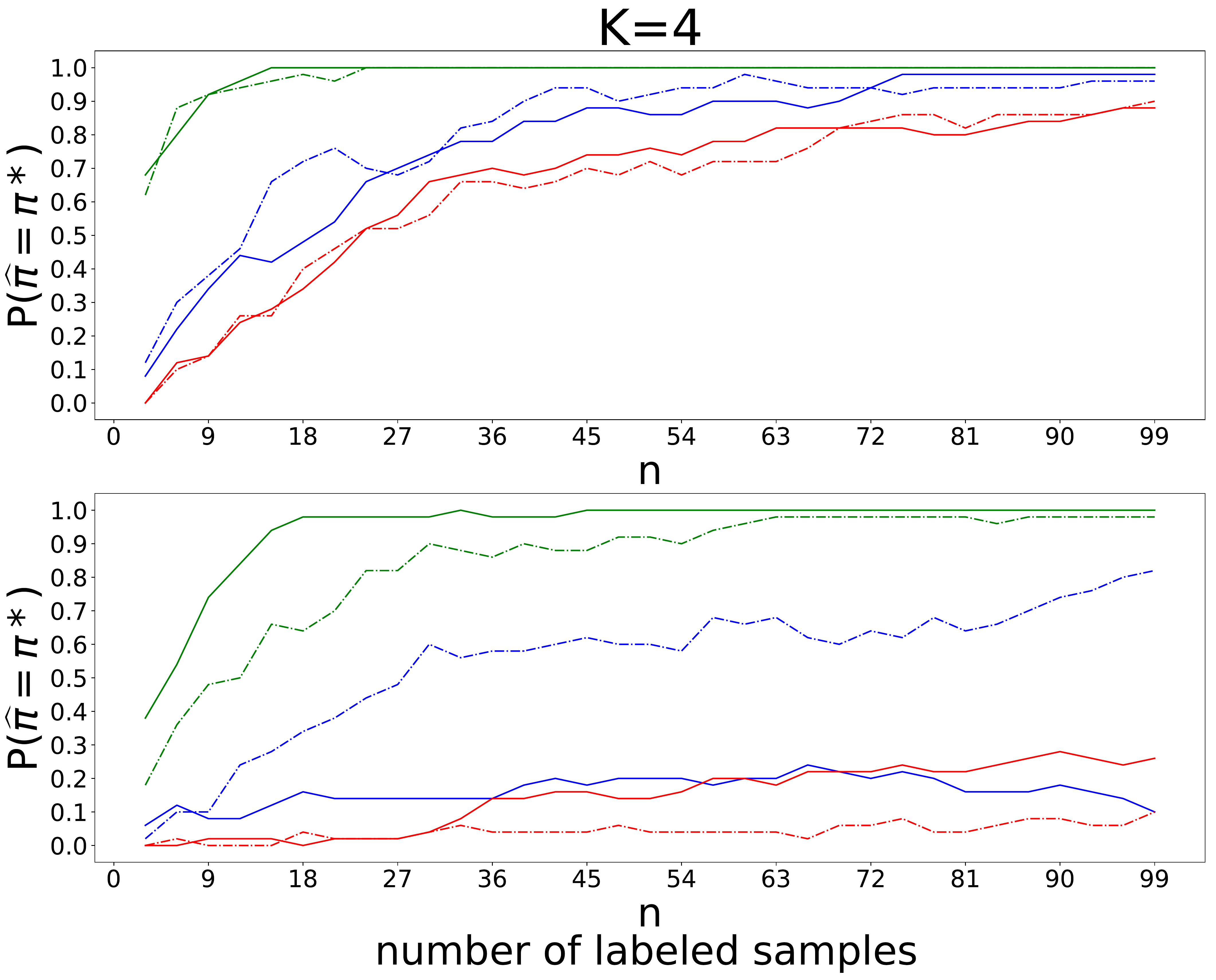}
  \caption{Mixture of Gaussian mixtures}
  \label{subfig:momog}
\end{subfigure}
~
\begin{subfigure}{.44\textwidth}
  \centering
  \vspace{1mm}
  \includegraphics[width=0.8\linewidth]{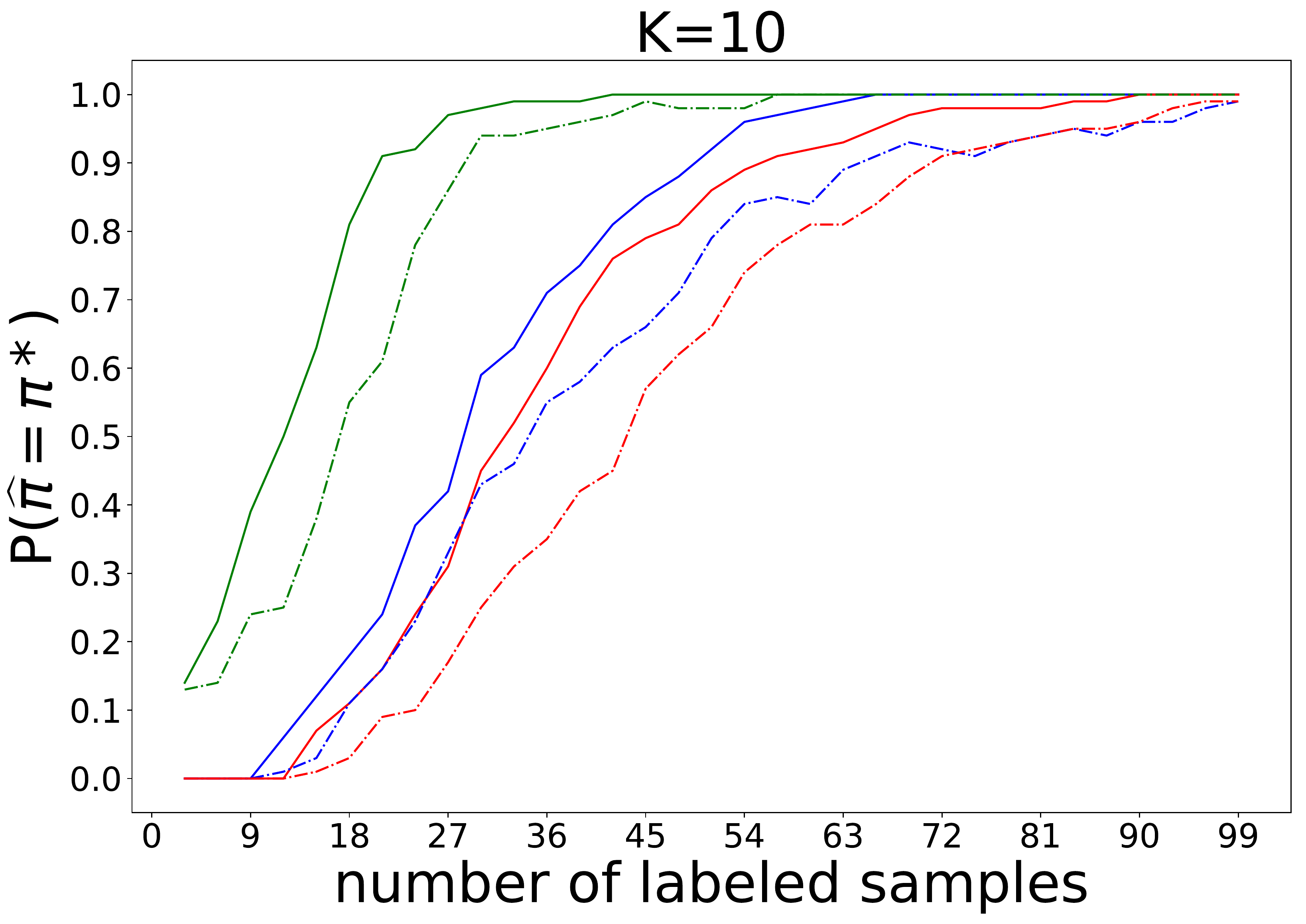}
  \caption{MNIST}
  \label{subfig:mnist}
\end{subfigure}
\caption{Performance of MLE (Hungarian - Green; Greedy - Blue) and MV (Red). Solid line and dashed line correspond to the performance when $\Lambda^* = \Lambda$ and  $\Lambda^* \neq \Lambda$, respectively. Columns correspond to the number of classes $K$; rows correspond to decreasing separation; e.g. the bottom rows in each figure are the least separated.}
\label{fig:mnist}
\end{center}
\end{figure}

\appendix

\section{Proofs}

\subsection{Proof of Theorem~\ref{thm:main:mle}}

\begin{proof}
    Denote a maximizer of the expected log-likelihood by $\identperm\in\argmax\E_{*}\nll(\perm;\mix)$ and define $\gap(\perm) = \E_{*}\nll(\identperm;\mix,X,Y) - \E_{*}\nll(\perm;\mix,X,Y)$. Note that $\gap(\perm)\ge\gap>0$ for all $\perm\ne\identperm$.
    Define $\mathcal{A}_{\perm}(t)=\{|\nll(\perm;\mix,X,Y)-\E_{*}\nll(\perm;\mix,X,Y))|<t\}$. 

    Then for any $t<\gap/2\le\gap(\pi)/2$, on the event $\cap_{\perm}\mathcal{A}_{\perm}(t)$ we have
    \begin{align*}
    \nll(\identperm;\mix,X,Y)
    &> \E_{*}\nll(\identperm;\mix,X,Y) - t \\
    &> \E_{*}\nll(\pi;\mix,X,Y) + \gap(\perm) - 2t \\
    &> \nll(\pi;\mix,X,Y)
    \quad\forall \pi \ne \identperm.
    \end{align*}
    
    Invoking Lemma~\ref{lem:conc:bound:perm} with $g_{k}(X,Y)=\log \wgt_{k}\denscmp_{k}(X,Y)$, we have 
    \begin{align*}
    \pr(\cap_{\pi}\mathcal{A}_{\pi}( t) ) 
    &= \pr \Big(\forall \pi, \Big|\frac{1}{n} \sum_{i=1}^n \nll(\identperm;\mix,\Xi,\Yi) - \E_{*}\nll(\identperm;\mix,\Xi,\Yi)\Big| \leq t
    \Big) \\
    & \geq 1 - 2K^2 \exp(-\inf_k \inf_b n_k \beta_{b}^* (t))
    \end{align*}
    
    Therefore, making the arbitrary choice of $t=\gap/3$,
    \begin{align*}
    \pr(\wh{\pi}=\wt{\pi}) &= \pr\big(\nll(\identperm;\mix,X,Y) > \nll(\perm;\mix,X,Y) \,\,\forall\pi\neq \wt{\pi}\big) \\
    & \geq 1 - 2 K^2\exp(-\inf_k \inf_b n_k \beta_{b}^* (\gap/3)).
    \end{align*}

    Since $\gap>0\implies \trueperm=\identperm$, the desired result follows.
\end{proof}

\subsection{Proof of Proposition~\ref{prop:gap:mle}}

\begin{proof}
    Let $p(x,y) = \truewgt_{\trueperm(y)} \truedenscmp_{\trueperm(y)} (x)$, $q(x,y)= \wgt_{\perm(y)}\denscmp_{\perm(y)} (x)$, so that
    \begin{align*}
        \E_{*}\nll(\trueperm;\truemix,X,Y) - \E_{*}\nll(\perm;\mix,X,Y) 
        &= \E_{*} \log(p(x,y)) - \E_{*} \log(q(x,y)) \\
        &= \int_{x} \sum_y  p(x,y)\log\frac{p(x,y)}{q(x,y)} dx \\
        &= \KL(p\,||\,q) \\
        & \geq 0.
    \end{align*}
    The equality holds if and only if $p(x,y) = q(x,y)$ holds for all $x,y$.
\end{proof}

\subsection{Proof of Theorem~\ref{thm:main:vote}}

\begin{proof}
We have
\begin{align*}
\pr(\wh{\pi}=\pi)
= \pr\big(\underbrace{\wh{\pi}(b)=b}_{\mathcal{E}_{b}} \,\,\forall b\in[K]\big)
= \pr\Big( 
\bigcap_{b=1}^{K}\mathcal{E}_{b}
\Big),
\end{align*}

\noindent
where
\begin{align*}
\mathcal{E}_{b}
= \Bigg\{\sum_{i=1}^{n}\ind(\Yi=b, \Xi\in\dec_{b}(\mix))
> \sum_{i=1}^{n}\ind(\Yi=j, \Xi\in\dec_{b}(\mix))
\quad \forall j\ne b
\Bigg\}.
\end{align*}

\noindent
Let $U_{bj}^{(i)}:=\ind(\Yi=j, \Xi\in\dec_{b}(\mix))$ so that $\chi_{bj}=\frac{1}{n_b}\sum_{i}U_{bj}^{(i)}$. It suffices to control the event
\begin{align}
\label{eq:main:event}
\Bigg\{\sum_{i=1}^{n}U_{bb}^{(i)}
> \sum_{i=1}^{n}U_{bj}^{(i)}
\quad \forall j\ne b
\Bigg\}
= \{\chi_{bb}>\chi_{bj} \,\, \forall j\ne b\}
\end{align}

\noindent
where $U_{j}^{(i)}\in\{0,1\}$ are i.i.d. random variables. Thus, we are interested in the probability $\pr(\chi_{bb}>\chi_{bj} \,\, \forall j\ne b)$. Note that
\begin{align*}
\E_{*}\chi_{bj}
= \frac{1}{n_b}\sum_{i=1}^{n}\E_{*} U_{bj}^{(i)}
= \frac{1}{n_b}\sum_{i:\Xi\in\dec_{b}}\pr(\Yi = j, \Xi\in\dec_{b}(\mix)).
\end{align*}

Define  
\begin{align}
\label{eq:defn:gap}
\gap_{bj}
:=\E_{*}\chi_{bb}-\E_{*}\chi_{bj}
\end{align}

\noindent
and $\mathcal{A}_{bj}(t)=\{|\chi_{bj}-\E_{*}\chi_{bj}|<t\}$. 
Then for any $t<\gap/2$, on the event $\cap_{j=1}^{K}\mathcal{A}_{bj}(t)$ we have
\begin{align*}
\chi_{bb}
> \E_{*}\chi_{bb} - t
> \E_{*}\chi_{bj} + \gap - 2t
> \chi_{bj}
\quad\forall j\ne b.
\end{align*}

In other words, making the arbitrary choice of $t=\gap/3$, we deduce
\[ \pr \big( \mathcal{E}_{b}^{c}\big) 
\leq \pr \Big( \bigcup_{j=1}^{K}\mathcal{A}_{j}(\Delta/3)^{c} \Big) 
\leq 2K\exp(-2n_b \gap^2/9)
\]

where we used Hoeffding's inequality to bound $\pr \big( \mathcal{A}_{j}(\gap/3)^{c}\big)$ for each $j$.

Thus
\begin{align*}
\pr\Big( 
\bigcap_{b=1}^{K}\mathcal{E}_{b}
\Big) 
&= 1 - \sum_{b=1}^K \pr \Big( \bigcup_{j=1}^{K}\mathcal{A}_{j}(\gap/3)^{c} \Big) \\
& \geq 1 - 2 K \sum_{b=1}^K \exp(-2n_b \gap^2/9)\\
& \geq 1 - 2 K^2 \exp\Big( \frac{-2\gap^{2}\min_b{n_b}}{9}\Big),
\end{align*}

\noindent
as claimed.
\end{proof}

\subsection{Proof of Proposition~\ref{prop:gap:vote}}

\begin{proof}
We have for any $j\ne b$,
\begin{align*}
\E_{*}\chi_{bb}(\truemix)
&= \frac{1}{n_b}\sum_{i=1}^{n}\E_{*}\ind(\Yi=b, \Xi\in\dec_{b}(\mix)) \\
&= \frac{1}{n_b}\sum_{i=1}^{n}\pr(\Yi=b, \Xi\in\dec_{b}(\mix)) \\
&= \frac{1}{n_b}\sum_{i=1}^{n}\pr(\Yi=b\given \Xi\in\dec_{b}(\mix))\pr(\Xi\in\dec_{b}(\mix)) \\
&> \frac{1}{n_b}\sum_{i=1}^{n}\pr(\Yi=j\given \Xi\in\dec_{b}(\mix))\pr(\Xi\in\dec_{b}(\mix)) \\
&= \E_{*}\chi_{bj}(\truemix).
\qedhere
\end{align*}
\end{proof}

\subsection{Proof of Corollaries~\ref{cor:sample:mle} and~\ref{cor:sample:vote}}

We prove Corollary~\ref{cor:sample:mle}; the proof of Corollary~\ref{cor:sample:vote} is similar with $n_{k}$ replaced by $m_{b}$ and $n_{0}$ in \eqref{eq:samp:mle} by $m_{0}$ in \eqref{eq:samp:vote}.

\begin{proof}
Using $p_{k}=1/K$ in Lemma~\ref{lem:multinomial:min}, we deduce for any $m>0$
\begin{align*}
\pr(\min_{k}n_{k}\ge m)
\ge 1 - K\exp\Big(-\frac{2K}{n}(n/K-m)^{2}\Big).
\end{align*}

\noindent
Thus, for any $\delta>0$, we have
\begin{align*}
n
\ge \frac{K}{2}\Big[\log(K/\delta) + 4m\Big]
\implies
\pr(\min_{k}n_{k}\ge m)
\ge 1-\delta.
\end{align*}

\noindent
The desired result follows from replacing $m$ with the lower bound on $n_{0}$ in \eqref{eq:samp:mle} and invoking Theorem~\ref{thm:main:mle}.
\end{proof}

\subsection{Proof of Theorem~\ref{thm:main:error}}

\begin{proof}
To avoid notational clutter, we will suppress the dependence on $m$ and $n$ in the following, so that $\estmix=\estmix_{m}$, $\estperm=\estperm_{m,n}$, $\estdec_{b}=\dec_{b}(\estmix_{m})$, and $\estclf=\estclf_{m,n}$. Write $\estdenscmp_{k}$ for the components of $\estmix$ and $\estwgt_{k}$ for the corresponding weights. Since $\estperm=\trueperm$, $\estdec_{b}$ corresponds to the decision region for label $\class_{b}$, and hence
\citep[see e.g. \sec2.5 in][]{devroye2013}
\begin{align}
\pr(\estclf(X)\ne Y)
\nonumber
&\le \pr(\trueclf(X)\ne Y)
+ \sum_{b}\pr(X\in\estdec_{b}\symdiff\truedec_{b}) \\
\label{eq:thm:main:error:proof:1}
&\le \pr(\trueclf(X)\ne Y)
+ \sum_{b}\int_{\mathcal{X}}|\estwgt_{b}\estdenscmp_{b}(x)-\truewgt_{b}\truedenscmp_{b}(x)|\,dx,
\end{align}

\noindent
where $\estdec_{b}\symdiff\truedec_{b}$ is the symmetric difference between the estimated and true decision regions.
Since $\wass(\estmix_{m},\mix)=O(r_{m})\to0$, we may assume without loss of generality that $\dTV(\estdenscmp_{b},\denscmp_{b})=O(r_{m})$ and $|\estwgt_{b}-\wgt_{b}|=O(r_{m})$. Focusing on the second quantity on the right hand side above, we have
\begin{align*}
\sum_{b}\int_{\mathcal{X}}|\estwgt_{b}\estdenscmp_{b}(x)-\truewgt_{b}\truedenscmp_{b}(x)|\,dx
&\le \sum_{b}\underbrace{\int_{\mathcal{X}}|\estwgt_{b}\estdenscmp_{b}(x)-\wgt_{b}\denscmp_{b}(x)|\,dx}_{(A)} + \sum_{b}\underbrace{\int_{\mathcal{X}}|\wgt_{b}\denscmp_{b}(x)-\truewgt_{b}\truedenscmp_{b}(x)|\,dx}_{(B)}.
\end{align*}

Now, for any $b$,
\begin{align*}
(A)
&\le |\estwgt_{b}-\wgt_{b}| + \wgt_{b}\,\dTV(\estdenscmp_{b},\denscmp_{b})
= O(r_{m}),
\end{align*}

\noindent
and invoking Lemma~\ref{lem:wass:bound},
\begin{align*}
(B)
&\le |\wgt_{b}-\truewgt_{b}| + \truewgt_{b}\,\dTV(\denscmp_{b},\truedenscmp_{b})
\le C(\truemix)\cdot \wass(\mix,\truemix).
\end{align*}

Thus
\begin{align*}
\sum_{b}\int_{\mathcal{X}}|\estwgt_{b}\estdenscmp_{b}(x)-\truewgt_{b}\truedenscmp_{b}(x)|\,dx
&\le K\big[O(r_{m}) + C(\truemix)\cdot \wass(\mix,\truemix)\big] \\
&\le Cr_{m} + C\cdot\wass(\mix,\truemix)
\end{align*}

\noindent
for some sufficiently large constant $C$ depending on $K$ and $\truemix$. Plugging this back into \eqref{eq:thm:main:error:proof:1} establishes the claim.
\end{proof}

\section{Additional lemmas}

\subsection{Lemma~\ref{lem:conc:bound:perm}}

For ease of notation in the following lemma, assume without loss of generality that $Y\in[K]$.

\begin{lemma}
\label{lem:conc:bound:perm}
Let $g_{1},\ldots,g_{K}$ be functions and $\psi_{k}(s)=\log\E_{*}\exp(s g_{k}(X,Y))$ be the log moment generating function of $g_{k}(X,Y)$. Then
\begin{align*}
\pr\Big(
\forall\pi : 
\frac1n\sum_{i=1}^{n}g_{\pi(Y_{i})}(X_{i}) - \E g_{\pi(Y_{i})}(X_{i})
\le t
\Big)
\ge 1 - K^{2}\exp(-\inf_{k}\inf_{b}n_{k}\psi_{b}^{*}(t)).
\end{align*}
\end{lemma}

\begin{proof}
Define $C_{k}:=\{i : Y_{i}=k\}$, $n_{k}:=|C_{k}|$, and note that
\begin{align*}
\{i:\pi(Y_{i}) = b\}
= \{i:Y_{i} = \pi^{-1}(b)\}
= C_{\pi^{-1}(b)}.
\end{align*}

\noindent
Then we have the following:
\begin{align*}
Z :=
\frac1n\sum_{i=1}^{n}g_{\pi(Y_{i})}(X_{i}) - \E g_{\pi(Y_{i})}(X_{i})
&= \frac1n\sum_{k=1}^{K}\sum_{i:\pi(Y_{i}) = b}g_{b}(X_{i}) - \E g_{b}(X_{i}) \\
&= \frac1n\sum_{b=1}^{K}\sum_{i\in C_{\pi^{-1}(b)}}g_{b}(X_{i}) - \E g_{b}(X_{i}) \\
&= \frac1n\sum_{b=1}^{K}n_{\pi^{-1}(b)}\Big\{\underbrace{\frac1{n_{\pi^{-1}(b)}}\sum_{i\in C_{\pi^{-1}(b)}}g_{b}(X_{i}) - \E g_{b}(X_{i})}_{:=\wt{Z}_{b}(\pi)}\Big\} \\
&= \sum_{b=1}^{K}\frac{n_{b}(\pi)}{n}\wt{Z}_{b}(\pi).
\end{align*}

\noindent
Now, for each $\pi$, $\wt{Z}_{b}(\pi)$ is just a sum over one of $K$ possible subsets of $[n]$, i.e. samples indices. To see this, define
\begin{align*}
Z_{b,k} := \frac1{n_{k}}\sum_{i\in C_{k}}g_{b}(X_{i}) - \E g_{b}(X_{i})
\end{align*}

\noindent
and note that $\wt{Z}_{b}(\pi)=Z_{b,\pi^{-1}(b)}$ for each $b$. It follows that 
\begin{align*}
Z 
= \sum_{b=1}^{K}\frac{n_{b}(\pi)}{n}\wt{Z}_{b}(\pi)
= \sum_{b=1}^{K}\frac{n_{\pi^{-1}(b)}}{n}Z_{b,\pi^{-1}(b)}
\end{align*}

\noindent
Chernoff's inequality implies $\pr(Z_{b,k}\ge t)\le\exp(-n_{k}\psi_{b}^{*}(t))$ for each $b$ and $k$, which implies that
\begin{align*}
\pr(\sup_{b,k} Z_{b,k}< t)
&= \pr\Big(\bigcap_{k}\bigcap_{b}\big\{Z_{b,k}< t\big\}\Big) \\
&\ge 1 - \pr\Big(\bigcup_{k}\bigcup_{b}\big\{Z_{b,k}< t\big\}^{c}\Big) \\
&\ge 1 - \sum_{k=1}^{K}\sum_{b=1}^{K}\pr(Z_{b,k} \ge t) \\
&\ge 1 - \sum_{k=1}^{K}\sum_{b=1}^{K}\exp(-n_{k}\psi_{b}^{*}(t)) \\
&\ge 1 - K^{2}\exp(-\inf_{k}\inf_{b}n_{k}\psi_{b}^{*}(t)).
\end{align*}

Now, if $\sup_{b,k} Z_{b,k}< t$, then 
\begin{align*}
Z 
= \sum_{b=1}^{K}\frac{n_{\pi^{-1}(b)}}{n}Z_{b,\pi^{-1}(b)}
< \sum_{b=1}^{K}\frac{n_{\pi^{-1}(b)}}{n}t
= t
\end{align*}

\noindent
since $\sum_{b}n_{b}/n=1$ and $\pi$ is a bijection. The desired result follows.
\end{proof}

\subsection{Lemma~\ref{lem:multinomial:min}}

The following lemma gives a precise bound on the minimum number of samples $n$ required to ensure $\min_{k}n_{k}\ge m$ from a generic multinomial sample with high probability:

\begin{lemma}
\label{lem:multinomial:min}
Let $Y_{i}$ be a multinomial random variable such that $\pr(Y_{i}=k)=p_{k}$ and define $n_{k}=\sum_{i=1}^{n}1(Y_{i}=k)$. Then for any $m>0$,
\begin{align*}
\pr(\min_{k}n_{k}\ge m)
\ge 1 - \sum_{k=1}^{K}\exp\Big(-\frac{2}{np_{k}}(np_{k}-m)^{2}\Big).
\end{align*}
\end{lemma}

\begin{proof}
By standard tail bounds on $n_{k}\sim\BinomialDist(n,p_{k})$, we have $\pr(n_{k}\le m)\le\exp(-2(np_{k}-m)^{2}/(np_{k}))$. Thus
\begin{align*}
\pr(\min_{k}n_{k}<m)
= \pr(\cup_{k=1}^{K}\{n_{k}< m\})
\le \sum_{k=1}^{K}\pr(n_{k}< m)
\le \sum_{k=1}^{K}\exp\Big(-\frac{2}{np_{k}}(np_{k}-m)^{2}\Big),
\end{align*}

\noindent
as claimed.
\end{proof}

\subsection{Lemma~\ref{lem:wass:bound}}

For any density $f\in L^{1}$, let $\delta_{f}$ denote the point mass concentrated at $f$, so that for any Borel subset $A\subset\alldens$, 
\begin{align*}
\delta_{f}(A)
= \begin{cases}
1, & f\in A \\
0, & f\notin A.
\end{cases}
\end{align*}

\begin{lemma}
\label{lem:wass:bound}
Let $\mix=\sum_{K=1}^{K}\wgt_{k}\delta_{\denscmp_{k}}$ and $\mix'=\sum_{K=1}^{K}\wgt_{k}'\delta_{\denscmp_{k}'}$. Then there is a constant $C=C(\mix',K)$ such that
\begin{align}
\label{lem:wass:bound:wgt}
\adjustlimits\sup_{j}\inf_{i}|\wgt_{i}-\wgt_{j}'|
&\le C\,\wass(\mix,\mix'), \\
\label{lem:wass:bound:denscmp}
\adjustlimits\sup_{j}\inf_{i}\dTV(f_{i}, f_{j}')
&\le C\,\wass(\mix,\mix').
\end{align}
\end{lemma}

\begin{proof}
The first inequality \eqref{lem:wass:bound:wgt} follows from Theorem~4 in \citet{gibbs2002}, and the second inequality \eqref{lem:wass:bound:denscmp} is standard.
\end{proof}

\bibliography{sslpermbib} 

\begin{thebibliography}{37}
\providecommand{\natexlab}[1]{#1}
\providecommand{\url}[1]{\texttt{#1}}
\expandafter\ifx\csname urlstyle\endcsname\relax
  \providecommand{\doi}[1]{doi: #1}\else
  \providecommand{\doi}{doi: \begingroup \urlstyle{rm}\Url}\fi

\bibitem[Aragam et~al.(2016)Aragam, Amini, and Zhou]{aragam2016}
B.~Aragam, A.~A. Amini, and Q.~Zhou.
\newblock Learning directed acyclic graphs with penalized neighbourhood
  regression.
\newblock arXiv:1511.08963, 2016.

\bibitem[Aragam et~al.(2018)Aragam, Dan, Ravikumar, and Xing]{aragam2018npmix}
B.~Aragam, C.~Dan, P.~Ravikumar, and E.~Xing.
\newblock Identifiability of nonparametric mixture models and bayes optimal
  clustering.
\newblock \emph{arXiv preprint}, arXiv:1802.04397, 2018.

\bibitem[Azizyan et~al.(2013)Azizyan, Singh, Wasserman, et~al.]{azizyan2013}
M.~Azizyan, A.~Singh, L.~Wasserman, et~al.
\newblock Density-sensitive semisupervised inference.
\newblock \emph{The Annals of Statistics}, 41\penalty0 (2):\penalty0 751--771,
  2013.

\bibitem[Barndorff-Nielsen(1965)]{barndorff1965}
O.~Barndorff-Nielsen.
\newblock Identifiability of mixtures of exponential families.
\newblock \emph{Journal of Mathematical Analysis and Applications}, 12\penalty0
  (1):\penalty0 115--121, 1965.

\bibitem[Blum and Mitchell(1998)]{blum1998}
A.~Blum and T.~Mitchell.
\newblock Combining labeled and unlabeled data with co-training.
\newblock In \emph{Proceedings of the eleventh annual conference on
  Computational learning theory}, pages 92--100. ACM, 1998.

\bibitem[Castelli and Cover(1995)]{castelli1995}
V.~Castelli and T.~M. Cover.
\newblock On the exponential value of labeled samples.
\newblock \emph{Pattern Recognition Letters}, 16\penalty0 (1):\penalty0
  105--111, 1995.

\bibitem[Castelli and Cover(1996)]{castelli1996}
V.~Castelli and T.~M. Cover.
\newblock The relative value of labeled and unlabeled samples in pattern
  recognition with an unknown mixing parameter.
\newblock \emph{IEEE Transactions on information theory}, 42\penalty0
  (6):\penalty0 2102--2117, 1996.

\bibitem[Chen(1995)]{chen1995}
J.~Chen.
\newblock Optimal rate of convergence for finite mixture models.
\newblock \emph{Annals of Statistics}, pages 221--233, 1995.

\bibitem[Collier and Dalalyan(2016)]{collier2016}
O.~Collier and A.~S. Dalalyan.
\newblock Minimax rates in permutation estimation for feature matching.
\newblock \emph{The Journal of Machine Learning Research}, 17\penalty0
  (1):\penalty0 162--192, 2016.

\bibitem[Cozman et~al.(2003)Cozman, Cohen, and Cirelo]{cozman2003}
F.~G. Cozman, I.~Cohen, and M.~C. Cirelo.
\newblock Semi-supervised learning of mixture models.
\newblock In \emph{Proceedings of the 20th International Conference on Machine
  Learning (ICML-03)}, pages 99--106, 2003.

\bibitem[Dai et~al.(2017)Dai, Yang, Yang, Cohen, and Salakhutdinov]{dai2017}
Z.~Dai, Z.~Yang, F.~Yang, W.~W. Cohen, and R.~R. Salakhutdinov.
\newblock Good semi-supervised learning that requires a bad gan.
\newblock In \emph{Advances in Neural Information Processing Systems}, pages
  6513--6523, 2017.

\bibitem[Devroye et~al.(2013)Devroye, Gy{\"o}rfi, and Lugosi]{devroye2013}
L.~Devroye, L.~Gy{\"o}rfi, and G.~Lugosi.
\newblock \emph{A probabilistic theory of pattern recognition}, volume~31.
\newblock Springer Science \&amp; Business Media, 2013.

\bibitem[Flajolet et~al.(1992)Flajolet, Gardy, and Thimonier]{flajolet1992}
P.~Flajolet, D.~Gardy, and L.~Thimonier.
\newblock Birthday paradox, coupon collectors, caching algorithms and
  self-organizing search.
\newblock \emph{Discrete Applied Mathematics}, 39\penalty0 (3):\penalty0
  207--229, 1992.

\bibitem[Flammarion et~al.(2016)Flammarion, Mao, and Rigollet]{flammarion2016}
N.~Flammarion, C.~Mao, and P.~Rigollet.
\newblock Optimal rates of statistical seriation.
\newblock \emph{arXiv preprint arXiv:1607.02435}, 2016.

\bibitem[Fogel et~al.(2013)Fogel, Jenatton, Bach, and d'Aspremont]{fogel2013}
F.~Fogel, R.~Jenatton, F.~Bach, and A.~d'Aspremont.
\newblock Convex relaxations for permutation problems.
\newblock In \emph{Advances in Neural Information Processing Systems}, pages
  1016--1024, 2013.

\bibitem[Gibbs and Su(2002)]{gibbs2002}
A.~L. Gibbs and F.~E. Su.
\newblock On choosing and bounding probability metrics.
\newblock \emph{International statistical review}, 70\penalty0 (3):\penalty0
  419--435, 2002.

\bibitem[Hall and Zhou(2003)]{hall2003}
P.~Hall and X.-H. Zhou.
\newblock Nonparametric estimation of component distributions in a multivariate
  mixture.
\newblock \emph{Annals of Statistics}, pages 201--224, 2003.

\bibitem[Heinrich and Kahn(2015)]{heinrich2015}
P.~Heinrich and J.~Kahn.
\newblock Minimax rates for finite mixture estimation.
\newblock \emph{arXiv preprint arXiv:1504.03506}, 2015.

\bibitem[Ho and Nguyen(2016{\natexlab{a}})]{ho2016}
N.~Ho and X.~Nguyen.
\newblock On strong identifiability and convergence rates of parameter
  estimation in finite mixtures.
\newblock \emph{Electronic Journal of Statistics}, 10\penalty0 (1):\penalty0
  271--307, 2016{\natexlab{a}}.

\bibitem[Ho and Nguyen(2016{\natexlab{b}})]{ho2016singularity}
N.~Ho and X.~Nguyen.
\newblock Singularity structures and impacts on parameter estimation in finite
  mixtures of distributions.
\newblock \emph{arXiv preprint arXiv:1609.02655}, 2016{\natexlab{b}}.

\bibitem[K{\"a}{\"a}ri{\"a}inen(2005)]{kaariainen2005}
M.~K{\"a}{\"a}ri{\"a}inen.
\newblock Generalization error bounds using unlabeled data.
\newblock In \emph{International Conference on Computational Learning Theory},
  pages 127--142. Springer, 2005.

\bibitem[Kingma et~al.(2014)Kingma, Mohamed, Rezende, and
  Welling]{kingma2014ssl}
D.~P. Kingma, S.~Mohamed, D.~J. Rezende, and M.~Welling.
\newblock Semi-supervised learning with deep generative models.
\newblock In \emph{Advances in Neural Information Processing Systems}, pages
  3581--3589, 2014.

\bibitem[Kuhn(1955)]{kuhn1955}
H.~W. Kuhn.
\newblock The hungarian method for the assignment problem.
\newblock \emph{Naval Research Logistics (NRL)}, 2\penalty0 (1-2):\penalty0
  83--97, 1955.

\bibitem[Lim and Wright(2014)]{lim2014}
C.~H. Lim and S.~Wright.
\newblock Beyond the birkhoff polytope: Convex relaxations for vector
  permutation problems.
\newblock In \emph{Advances in Neural Information Processing Systems}, pages
  2168--2176, 2014.

\bibitem[Newman(1960)]{newman1960}
D.~J. Newman.
\newblock The double dixie cup problem.
\newblock \emph{The American Mathematical Monthly}, 67\penalty0 (1):\penalty0
  58--61, 1960.

\bibitem[Niyogi(2013)]{niyogi2013}
P.~Niyogi.
\newblock Manifold regularization and semi-supervised learning: Some
  theoretical analyses.
\newblock \emph{The Journal of Machine Learning Research}, 14\penalty0
  (1):\penalty0 1229--1250, 2013.

\bibitem[Pananjady et~al.(2016)Pananjady, Wainwright, and
  Courtade]{pananjady2016}
A.~Pananjady, M.~J. Wainwright, and T.~A. Courtade.
\newblock Linear regression with an unknown permutation: Statistical and
  computational limits.
\newblock In \emph{Communication, Control, and Computing (Allerton), 2016 54th
  Annual Allerton Conference on}, pages 417--424. IEEE, 2016.

\bibitem[Rigollet(2007)]{rigollet2007}
P.~Rigollet.
\newblock Generalization error bounds in semi-supervised classification under
  the cluster assumption.
\newblock \emph{Journal of Machine Learning Research}, 8\penalty0
  (Jul):\penalty0 1369--1392, 2007.

\bibitem[Seeger(2000)]{seeger2000}
M.~Seeger.
\newblock Learning with labeled and unlabeled data.
\newblock Technical report, 2000.

\bibitem[Singh et~al.(2009)Singh, Nowak, and Zhu]{singh2009}
A.~Singh, R.~Nowak, and X.~Zhu.
\newblock Unlabeled data: Now it helps, now it doesn't.
\newblock In \emph{Advances in neural information processing systems}, pages
  1513--1520, 2009.

\bibitem[Teicher(1961)]{teicher1961}
H.~Teicher.
\newblock Identifiability of mixtures.
\newblock \emph{The annals of Mathematical statistics}, 32\penalty0
  (1):\penalty0 244--248, 1961.

\bibitem[Teicher(1963)]{teicher1963}
H.~Teicher.
\newblock Identifiability of finite mixtures.
\newblock \emph{The annals of Mathematical statistics}, pages 1265--1269, 1963.

\bibitem[Teicher(1967)]{teicher1967}
H.~Teicher.
\newblock Identifiability of mixtures of product measures.
\newblock \emph{The Annals of Mathematical Statistics}, 38\penalty0
  (4):\penalty0 1300--1302, 1967.

\bibitem[van~de Geer and B{\"u}hlmann(2013)]{geer2013}
S.~van~de Geer and P.~B{\"u}hlmann.
\newblock $\ell_0$-penalized maximum likelihood for sparse directed acyclic
  graphs.
\newblock \emph{Annals of Statistics}, 41\penalty0 (2):\penalty0 536--567,
  2013.

\bibitem[Wasserman and Lafferty(2008)]{wasserman2008}
L.~Wasserman and J.~D. Lafferty.
\newblock Statistical analysis of semi-supervised regression.
\newblock In \emph{Advances in Neural Information Processing Systems}, pages
  801--808, 2008.

\bibitem[Yakowitz and Spragins(1968)]{yakowitz1968}
S.~J. Yakowitz and J.~D. Spragins.
\newblock On the identifiability of finite mixtures.
\newblock \emph{The Annals of Mathematical Statistics}, pages 209--214, 1968.

\bibitem[Zhu et~al.(2003)Zhu, Ghahramani, and Lafferty]{zhu2003}
X.~Zhu, Z.~Ghahramani, and J.~D. Lafferty.
\newblock Semi-supervised learning using gaussian fields and harmonic
  functions.
\newblock In \emph{Proceedings of the 20th International conference on Machine
  learning (ICML-03)}, pages 912--919, 2003.

\end{thebibliography}
\bibliographystyle{abbrvnat}
\end{document}